\documentclass[final,onefignum,onetabnum]{siamonline220329}

\usepackage{lipsum}
\usepackage{amsfonts}
\usepackage{graphicx}
\usepackage{epstopdf}
\usepackage{algorithm}
\usepackage{algpseudocode}
\ifpdf
  \DeclareGraphicsExtensions{.eps,.pdf,.png,.jpg}
\else
  \DeclareGraphicsExtensions{.eps}
\fi

\allowdisplaybreaks

\usepackage{enumitem}
\setlist[enumerate]{leftmargin=.5in}
\setlist[itemize]{leftmargin=.5in}

\newsiamremark{remark}{Remark}
\newsiamremark{hypothesis}{Hypothesis}
\crefname{hypothesis}{Hypothesis}{Hypotheses}
\newsiamthm{claim}{Claim}

\headers{A randomized algorithm to solve reduced rank operator regression}{G. Turri, V. Kostic, P. Novelli, and M. Pontil}

\title{A randomized algorithm to solve reduced rank operator regression
\thanks{%
\funding{The work is partially funded by the Italian Ministry of University and Research (MUR) through the project PNRR MUR Project PE000013 CUP J53C22003010006 "Future Artificial Intelligence Research (FAIR)" and was carried out within the framework of the project "RAISE - Robotics and AI for Socio-economic Empowerment”, both supported by the European Union – NextGenerationEU. However, the views and opinions expressed are those of the authors alone and do not necessarily reflect those of the European Union or the European Commission. Neither the European Union nor the European Commission can be held responsible for them.}
}
}

\author{Giacomo Turri\footnotemark[2]\ \footnotemark[3]\ \footnotemark[4]\ \footnotemark[7]
\and Vladimir Kostic\footnotemark[2]\ \footnotemark[5]\ \footnotemark[7]
  \and Pietro Novelli\footnotemark[2]\ \footnotemark[7]
\and Massimiliano Pontil\footnotemark[2]\ \footnotemark[6]}

\usepackage{amsopn}

\usepackage[textwidth=2.0cm, textsize=tiny]{todonotes}
\makeatletter 
\@mparswitchfalse%
\makeatother
\normalmarginpar

\usepackage{pifont}%

\DeclareMathOperator*{\argmin}{\ensuremath{\text{\rm arg\,min}}}

\DeclareMathOperator*{\range}{\ensuremath{\text{\rm Im}}}

\DeclareMathOperator{\Ker}{\ensuremath{\text{\rm Ker}}}

\DeclareMathOperator{\tr}{\ensuremath{\text{\rm tr}}}
\DeclareMathOperator{\Diag}{\ensuremath{\text{\rm diag}}}
\DeclareMathOperator*{\rank}{\ensuremath{\text{\rm rank}}}

\providecommand{\norm}[1]{\lVert#1\rVert}
\providecommand{\SVDr}[1]{[\![#1]\!]_r}

\newcommand{\scalarpH}[1]{{\langle #1\rangle}_{\spH}}
\newcommand{\scalarpG}[1]{{\langle #1\rangle}_{\spG}}
\newcommand{\R}{\mathbb R}

\newcommand{\Nc}{\mathcal N}
\newcommand{\OO}[1]{\boldsymbol{\mathsf{#1}}}
\newcommand{\EE}{\ensuremath{\mathbb E}}

\newcommand{\Id}{I}
\newcommand{\Idop}{\OO{I}}
\newcommand{\Data}{\mathcal{D}_n}

\newcommand{\EEstim}{\Estim}  %
\newcommand{\Estim}{\OO{A}}  %
\newcommand{\ECx}{\OO{C}} %
\newcommand{\ECy}{\OO{D}} %
\newcommand{\ECxy}{\OO{T}}  %
\newcommand{\ECyx}{\ECxy^*} 
\newcommand{\ECreg}{\ECx_\reg}
\newcommand{\EZ}{\OO{Z}} %
\newcommand{\ES}{\OO{S}} %

\newcommand{\X}{\mathcal{X}} %
\newcommand{\Y}{\mathcal{Y}} %
\newcommand{\ERisk}{\mathcal{R}^\reg} %
\newcommand{\HS}[1]{{\rm{HS}}\left(#1\right)} %
\newcommand{\hnorm}[1]{\norm{#1}_{\rm{HS}}}
\newcommand{\proj}[1]{\pi({#1})}

\newcommand{\sketch}{\Omega}
\newcommand{\wsketch}{\sketch_{p}}
\newcommand{\reg}{\gamma}

\newcommand{\refun}{\eta}
\newcommand{\lefun}{\xi}

\newcommand{\spH}{\mathcal{H}}
\newcommand{\spG}{\mathcal{G}}
\newcommand{\fH}{\phi}
\newcommand{\fG}{\psi}
\newcommand{\HSr}{{\rm{HS}}_r({\spH, \spG})}

\newcommand{\Creg}{\ECx_\reg}

\newcommand{\Kx}{K}
\newcommand{\Kreg}{K_{\reg}}
\newcommand{\Ky}{L}

\newtheorem{fact}{Fact}
\ifpdf
\hypersetup{
  pdftitle={A Randomized Algorithm
to Solve Reduced Rank Operator Regression},
  pdfauthor={G. Turri, V. Kostic, P. Novelli, and M. Pontil}
}
\fi

\begin{document}

\maketitle

\newcommand*{\myfnsymbolsingle}[1]{%
  \ensuremath{%
    \ifcase#1%
    \or %
      *%
    \or %
      \dagger
    \or %
      \ddagger
    \or %
      \mathsection
    \or %
      \mathparagraph
    \or %
      \|
    \or %
      \#
    \else %
      \@ctrerr  
    \fi
  }%
}   
\makeatother

\newcommand*{\myfnsymbol}[1]{%
  \myfnsymbolsingle{\value{#1}}%
}

\renewcommand{\thefootnote}{\myfnsymbol{footnote}}

\footnotetext[2]{\sffamily{Computational Statistics and Machine Learning (CSML) unit, Istituto Italiano di Tecnologia, Genoa, Italy \email{(giacomo.turri@iit.it; vladimir.kostic@iit.it; pietro.novelli@iit.it; massimiliano.pontil@iit.it)}}}
\footnotetext[3]{\sffamily{Unit for Visually Impaired People (U-VIP), Istituto Italiano di Tecnologia, Genoa, Italy}}
\footnotetext[4]{\sffamily{RAISE ecosystem, Genoa, Italy}}
\footnotetext[5]{\sffamily{University of Novi Sad, Novi Sad, Serbia}}
\footnotetext[6]{\sffamily{Department of Computer Science and 
UCL Centre for Artificial Intelligence, University College London, United Kingdom}}
\footnotetext[7]{\sffamily{These authors contributed equally to this work}}

\renewcommand{\thefootnote}{\arabic{footnote}}

\begin{abstract}
We present and analyze an algorithm designed for addressing vector-valued regression problems involving possibly infinite-dimensional input and output spaces.~The algorithm is a randomized adaptation of {\em reduced rank regression}, a technique introduced in \cite{izenman1975reduced} and extended to infinite dimensions in \cite{Mukherjee2011} to optimally learn a low-rank vector-valued function (i.e. an operator) between sampled data via regularized empirical risk minimization with rank constraints.~We propose Gaussian sketching techniques both for the primal and dual optimization objectives, yielding {\em Randomized Reduced Rank Regression} (R$^4$) estimators that are efficient and accurate. 
For each of our R$^4$ algorithms we prove that the resulting regularized empirical risk is, in expectation w.r.t. randomness of a sketch, arbitrarily close to the optimal value when hyper-parameteres are properly tuned.~Numerical expreriments illustrate the tightness of our bounds and show advantages in two distinct scenarios:~(i) solving a vector-valued regression problem using synthetic and large-scale neuroscience datasets, and (ii) regressing the Koopman operator of a nonlinear stochastic dynamical system.
\end{abstract}



\section{Introduction}
In this work we study the problem of {linear operator regression} $\Estim:\spH \to \spG$ {between} Hilbert spaces $\spH$ {and} $\spG$. In particular, we assume that $\spH$ and $\spG$ are Reproducing Kernel Hilbert Spaces~\cite{Aronszajn1950} associated to feature maps $\fH\colon \X \to \spH$ and $\fG\colon \Y \to \spG$ on the data space $\X$, respectively. We aim to find low rank solutions to the problem
\begin{equation}\label{eq:ORP}
    \min_{\Estim \in \HS{\spH}} \EE_{(x,y)\sim\rho} \norm{\fG(y) - \Estim\fH(x)}^{2}_{\spH}.
\end{equation}
Here $\HS{\spH}$ is the space of Hilbert-Schmidt operators on $\spH$ and $\rho$ is an unknown probability distribution on $\X \times \Y$.%

{Whenever both $\fH$ and $\fG$ are the identity functions,~\eqref{eq:ORP} recover the classical reduced rank regression objective~\cite{izenman1975reduced}. On the other hand, the full-fledged kernelized version} of~\eqref{eq:ORP} are of interest in the machine learning community, and especially for structured prediction~\cite{
Ciliberto2020}, conditional mean embeddings~\cite{Song2009} and general linear operator regression~\cite{Mollenhauer2022}. {Furthermore,} as amply discussed in~\cite{Kostic2022}, problems of the form~\eqref{eq:ORP} naturally arise {also} in the context of dynamical systems learning, where accurate low-rank solutions provide a succinct and interpretable tool to describe nonlinear dynamical systems {via Koopman/Transfer operators}.

When a dataset $\Data = \{(x_i,y_i)\}_{i\in[n]}$ sampled from $\rho$ is available,~\eqref{eq:ORP} is approximately solved via (regularized) empirical risk minimization (RERM). In this work we consider the following empirical minimization problem
\begin{equation}\label{eq:RRR}
\EEstim_{r, \reg} := \argmin_{\Estim \in \HSr} \frac{1}{n}\sum_{i=1}^{n}\norm{\fG(y_{i}) - \Estim\fH(x_i)}^{2}_{\spG} + \reg\norm{\Estim}^{2}_{{\rm HS}} := \ERisk(\Estim),
\end{equation}
where $\HSr = \{\Estim\in \HS{\spH,\spG}\,\vert\, \rank(\Estim)\leq r\}$, while $r\geq1$ and $\reg>0$ are the rank constraint and regularization parameter, respectively. The estimator \eqref{eq:RRR} has been originally proposed in the finite dimensional setting in \cite{izenman1975reduced} and recently extended to infinite dimensional spaces in \cite{Mukherjee2011, Kostic2022}. 

Notice that in~\eqref{eq:RRR} two concurrent regularization schemes are considered, the first being a hard constraint on the rank of the estimator and the other being the classical Tikhonov regularization. In this work we provide an algorithm to {\em efficiently} solve \eqref{eq:RRR}, eliminating the need to solve large (and potentially non-symmetric) generalized eigenvalue problems arising when dealing with large scale datasets.
To that end, we address the {\em fixed-design} setting, ignoring the randomness of the data $\Data$ and any statistical-learning-related matters.

\subsubsection*{Related work} Low rank regression models have wide applications. For instance they have been used for time series analysis and forecasting \cite{chen2021autoregressive,reinsel2022}, in chemometrics \cite{Mukherjee2011}, in multitask learning \cite{argyriou2008convex,cella2023multi}, in structure prediction \cite{luise2019}, and to compress parameters in neural networks \cite{idelbayev2020}, among others. We also refer to \cite{brogat2022,dinuzzo2011,rabusseau2016} and references therein for related work on low rank regression problems. 

\subsection*{Structure of the paper}
We introduce the background and setting in Section \ref{sec:background}, and define our randomized adaptation of the reduced rank regression estimator in Section \ref{sec:alg}. In Section \ref{sec:bounds} we provide error bounds on the risk of the estimators, while in Section \ref{sec:experiments} we preset numerical experiments on synthetic as well as large-scale neuroscience datasets.

\subsection*{Notation}
Throughout the paper operators are denoted with bold-face sans-serif fonts (e.g. $\OO{A}$) and matrices with roman fonts (e.g. $A$). For any operator (or matrix) $\OO{A}$, the orthogonal projector onto $\range(\OO{A})$ is $\proj{\OO{A}}$. The composition $\OO{A}\circ M$ between $\OO{A}:\R^{n}\to \mathcal{V}$, $\mathcal{V}$ being a possibly infinite dimensional linear space, and $M \in \R^{n\times k}$ is shortened to $\OO{A}M$. The $r$-truncated SVD of Hilbert-Schmidt operators and matrices is denoted by $\SVDr{\cdot}$.

\section{Background and preliminaries}\label{sec:background}

Given the data $\Data := (x_i, y_i)_{i=1}^{n}$, the {\em sampling operators} are objects relating the spaces $\spH$ and $\spG$ to the empirical distribution of data. They are particularly handy in the analysis of the optimization objective~\eqref{eq:RRR} in either its primal or (convex) dual form~\cite{Boyd2004}. The sampling operators for inputs $\ES \colon \spH \to \R^{n}$ and outputs  $\EZ \colon \spG \to \R^{n}$ are defined, respectively, as 
\begin{align*}
    \ES h := \tfrac{1}{\sqrt{n}}[ \scalarpH{h, \fH(x_{i})}]_{i \in[n]}, \quad \EZ g := \tfrac{1}{\sqrt{n}}[ \scalarpG{g, \fG(y_{i})}]_{i \in[n]}, \qquad \text{ for all } h \in \spH, g \in \spG.
\end{align*}
It is readily verified that their adjoints are given by 
\begin{align*}
\ES^*v = \tfrac{1}{\sqrt{n}}\sum_{i\in[n]}v_i\fH(x_i), \quad \EZ^* u = \tfrac{1}{\sqrt{n}}\sum_{i\in[n]}u_i\fG(y_i),  \qquad \text{ for all } u,v \in \R^{n}.
\end{align*}
Composing sampling operators one obtains
\begin{align*}
        \ECx := \ES^*\ES = \tfrac{1}{n}\sum_{i \in[n]} \fH(x_{i})\otimes\fH(x_{i}), \quad \ECy := \EZ^*\EZ = \tfrac{1}{n}\sum_{i \in[n]} \fG(y_{i})\otimes\fG(y_{i}),
    \end{align*} 
the input and output {\em covariance operators}, respectively, and 
    \begin{align*}
        \Kx := \ES\ES^* = n^{-1}[k(x_i,x_j)]_{i,j\in[n]}, \quad \Ky := \EZ\EZ^* =  n^{-1}[\ell(y_i,y_j)]_{i,j\in[n]},
        \end{align*}
the input and output {\em kernel matrices}, respectively. Moreover, the cross-covariance operator $\ECxy \colon \spG \to \spH$ is given by the formula    \begin{align*}
        \ECxy := \ES^*\EZ = \tfrac{1}{n} \sum_{i \in[n]} \fH(x_{i})\otimes\fG(y_{i}).
    \end{align*}
    
We conclude with some results on the rank-constrained empirical risk minimization~\eqref{eq:RRR}, omitting the proofs since they are minimal modification of the results in~\cite{Kostic2022}.
\begin{proposition}\label{prop:RRR_recap}
For the empirical risk minimization~\eqref{eq:RRR} the following results hold.
    \begin{itemize}
        \item[(1)] For any $\EEstim\in\HS{\spH,\spG}$, the risk defined in~\eqref{eq:RRR} can be decomposed as 
        \begin{equation}\label{eq:RRR_risk}
        \ERisk(\EEstim) = \tr(\ECy) - \hnorm{\ECyx \ECx_\reg^{-\frac{1}{2}}}^2 + \hnorm{\EEstim\ECx_\reg^{\frac{1}{2}}  - \ECyx\ECx_\reg^{-\frac{1}{2}}}^2.
        \end{equation}
        \item[(2)] The minimizer $\EEstim_{r, \reg}$ of \eqref{eq:RRR} is
        \begin{equation}\label{eq:RRR_solution}
            \EEstim_{r, \reg} = \SVDr{\ECyx\ECreg^{-\frac{1}{2}}}\ECreg^{-\frac{1}{2}}, 
        \end{equation}
        where $ \Creg := \ECx + \reg \Idop$ is the {\em Tikhonov-regularized} input covariance. 
        \item[(3)] The optimal value of the regularized risk is
        \begin{equation}\label{eq:RRR_optimal_risk}
        \ERisk(\EEstim_{r, \reg}) = \tr(\ECy) - \sum_{i\in[r]}\sigma^2_{i},
        \end{equation}
        where $\sigma_1\geq\cdots\geq\sigma_r$ are the leading singular values of $\ECyx \ECx_\reg^{-\frac{1}{2}}$.
    \end{itemize}
\end{proposition}        

For finite dimensional $\spH$ and $\spG$, since $\ECreg$ and $\ECxy$ are matrices, the solution $ \EEstim_{r, \reg} $ can be computed by directly evaluating the covariance matrices. Otherwise, one needs to translate the computation to the convex dual using the sampling operators. This is summarized in the following result.

\begin{proposition}\label{prop:RRR_comp}
Denoting the regularized kernel matrix of inputs as $\Kreg:=\Kx+\reg \Id$, the operator $\EEstim_{r, \reg}$ can be represented in the following ways:
\begin{itemize}
\item[(1)] Primal formulation: $ \EEstim_{r, \reg} = \ECxy^*\OO{H}_r$, where $\OO{H}_r:=\sum_{i\in[r]} \scalarpH{h_i,\Creg\,h_i}^{-1}\, h_i\otimes h_i$ with $h_i\in\spH$ being the $r$ leading eigenvectors of the generalized eigenvalue problem (GEP) 
\begin{equation}\label{eq:GEP_primal}
        \ECxy \ECxy^*  h_i = \sigma_i^2 \ECreg h_i,\quad i\in[r],
        \end{equation}
\item[(2)] Dual formulation: $ \EEstim_{r, \reg} = \EZ^* \widehat{U}_r \widehat{V}_r^\top \ES$,  where $\widehat{U}_r = \Kx  \widehat{V}_r$ and $\widehat{V}_r \in\R^{n\times r} $ is the matrix whose columns are the $r$ leading eigenvectors solving the GEP
        \begin{equation}\label{eq:GEP_dual}
        \Ky \Kx  \widehat{v}_i = \sigma_i^2 K_\reg \widehat{v}_i,\quad i\in[r]
        \end{equation} 
        normalized as $\widehat{v}_i^\top \Kx  K_\reg \widehat{v}_i = 1$, $i \in [r]$.
\end{itemize}
\end{proposition}

In typical situations when the dimension of the RKHS spaces is very large and one has many samples,  the methods above are computationally intensive since large scale symmetric GEP \eqref{eq:GEP_primal}, or even non-symmetric GEP~\eqref{eq:GEP_dual} needs to be solved. In the latter case, the situation is often aggravated by $\Ky$ and $\Kx$ being ill-conditioned. In these common scenarios, classical approaches such as the implicitly restarted Arnoldi method, might exhibit slow convergence. In this work we propose an alternative numerical approach by computing the SVD appearing in $\EEstim_{r, \reg} = \SVDr{\ECyx\ECreg^{-\frac{1}{2}}}\ECreg^{-\frac{1}{2}}$ via a randomized algorithm. 

Relying on the classical theory of randomized SVD algorithms for matrices presented in~\cite{Martinsson2020}, we propose a novel randomized algorithm specialized in solving the RERM~\eqref{eq:RRR}. We also provide error bounds in expectation and show the performance improvement over the current GEP solver. Following Proposition~\ref{prop:RRR_comp}, we develop our randomized algorithms both for the primal and dual formulations of~\eqref{eq:RRR}. To that end, we rely on the following standard results on linear operators, see e.g.~\cite{Kato1966}.%
\begin{fact}\label{lm:operator_EVD}
    Let $\OO{M}:\mathcal{V}\to\mathcal{V} = \OO{A}^{*}\OO{B}$ with $\OO{A},\OO{B}:\mathcal{V}\to \R^{n}$, $\mathcal{V}$ being a generic vector space. The eigenvectors $h_{i}$ of $\OO{M}$ corresponding to non-null eigenvalues $\lambda_{i} \neq 0$ are of the form $h_{i} = \OO{A}^{*}w_{i}$ with $w_{i} \in \R^{n}$ satisfying the matrix eigenvalue equation 
    \begin{equation*}
        \OO{B}\OO{A}^{*}w_{i} = \lambda_{i}w_{i}.
    \end{equation*}
\end{fact}

\begin{fact} Let $f$ be a continuous spectral function and $\OO{A}$ be a compact linear operator. Then, it holds $ \OO{A}^*f(\OO{A}\OO{A}^*) = f(\OO{A}^*\OO{A})\OO{A}$. In particular, for any $\reg > 0$ one has  $\OO{A}^*\left(\OO{A}\OO{A}^* + \reg\Id\right)^{-1} = \left(\OO{A}^{*}\OO{A} + \reg\Id\right)^{-1} \OO{A}$.
\end{fact}
\begin{fact}\label{fct:projector_onto_range}
    Let $A \in \R^{n\times m}$, $\OO{M}:\R^{n} \to \mathcal{V}$, $\mathcal{V}$ being a generic vector space. Then, the projector onto the range of $\OO{B} A$ is then given by $\proj{\OO{M} A} = \OO{M} W \OO{M}^*$, where $W = A(A^* \OO{M}^*\OO{M} A)^{\dagger} A^*\in\R^{n\times n}$.
\end{fact}
\begin{fact}\label{fct:GEP}
Let $A,B\in\R^{n\times n}$ be two symmetric positive semi-definite matrices. Then $(\lambda,q)\in\R_+ \times \R^{n}$ is a finite eigenpair of $A^\dagger B$ if and only if $B q  = \lambda A q$ and $q\in\Ker(A)^\perp$.
\end{fact}

Finally, we conclude this section with a simple preliminary result used in proving the error bounds. It relates SVD of an operator $\OO{B}$ to the SVD of the appropriate matrix $B$, and is slightly different form can also be found in \cite{Mollenhauer2020}.  %

\begin{proposition}\label{prop:operator_SVD}
    Let $\OO{B}=\ES^*W\EZ :\spG\to\spH $ and $B = \Kx^{\frac{1}{2}} W \Ky^{\frac{1}{2}}\in\R^{n\times n}$. Then $\rank(\OO{B})=\rank(B)=:r$ and there exist $U, V\in\R^{n\times r}$ and a positive diagonal matrix $\Sigma \in\R^{r\times r}$ such that $\OO{B}=\ES^*U\Sigma V^*\EZ $ and $B = \Kx^{\frac{1}{2}} U\Sigma V^* \Ky^{\frac{1}{2}}\in\R^{n\times n}$. Moreover, the columns of $U$ form a $\Kx$-orthonormal system, whereas the columns of $V$ are $\Ky$-orthonormal.
\end{proposition}
\begin{proof}
First observe that $\rank(\OO{B}) = \rank(\OO{B}^*\OO{B})$ and $\rank(B) = \rank(B^*B)$. So,   due to Lemma \ref{lm:operator_EVD} we have
\[
\rank(\OO{B}) = \rank(\EZ^*W^*\ES \ES^* W \EZ) = \rank(\EZ \EZ^*W^*\Kx W ) = \rank(\Ky^{\frac{1}{2}} W^* \Kx W \Ky^{\frac{1}{2}} W) = \rank(B).
\]
Now, let the reduced SVD of $\OO{B}$ be $\OO{B} := \sum_{i\in[r]}\sigma_i\, \lefun_i \otimes \refun_i$, where $(\lefun_i)_{i\in[r]}\in\spH$ and $(\refun_i)_{i\in[r]}\in\spG$ are orthonormal systems and $\sigma_1\geq\ldots\geq\sigma_r>0$. Then we have that $\OO{B}^*\OO{B}\refun_i = \sigma_i^2\refun_i$, $i\in[r]$, and again using Lemma \ref{lm:operator_EVD}, we have that for every $i\in[r]$, $\refun_i = \EZ^*v_i$ where $W^*\Kx W \Ky v_i = \sigma_i^2 v_i$. But then $(v_i)_{i\in[r]}\in\R^{n}$ is a $\Ky$-orthonormal system, and therefore $(\Ky^{\frac{1}{2}} v_i)_{i\in[r]}\in\R^{n}$ is an orthonormal system. However, since
\[
B^*B (\Ky^{\frac{1}{2}} v_i) = \Ky^{\frac{1}{2}} W^*\Kx W \Ky v_i = \sigma_i^2  \Ky^{\frac{1}{2}}  v_i,
\]
we conclude that the right singular vectors of $B$ are given by $(\Ky^{\frac{1}{2}} v_i)_{i\in[r]}\in\R^{n}$. To conclude we use that left singular functions are given by $\lefun_i = \frac{1}{\sigma_i}\OO{B}\refun_i$, $i\in[r]$, and the claim readily follows. 
\end{proof}

\section{Randomized reduced rank regression}\label{sec:alg}

The general idea behind the randomized SVD algorithm for matrices is to approximate the space $\range(B)$ of a target matrix $B$ by a (small) subspace $\range(M)$ for some random matrix $M$. One way to construct it is via {\em powered randomized rangefinder} (see \cite[Algorithm 9]{Martinsson2020}) as $M=(BB^*)^p Y$, $Y$ being a low-rank Gaussian matrix. The truncated SVD of $B$ is then approximated as $\SVDr{B}\approx \SVDr{\proj{M}B}$. In this work, we pursue the same ideas to approximate the $r$-truncated SVD of the operator $\OO{B} := \ECreg^{-\frac{1}{2}}\ECxy$ appearing, in its adjoint form, in $\EEstim_{r, \reg}$. We point out that a na\^ive extension of randomized SVD algorithms for matrices is not efficient in the present case, as the presence of the square-root of the positive operator $\Creg$ appearing in $\OO{B}$ is too expensive to be computed directly. Thus, we propose a method that is tailored to avoid any explicit calculation of the square-root of positive operators. Square-roots of operators appear in the derivation but will always cancel out in the final results. We now detail every step of the proposed algorithm: %
\begin{itemize}
    \item[(1)] We define the sketching operator $\OO{Y} := \ECreg^{-\frac{1}{2}}\OO{\sketch}$. Here $\OO{\sketch} \colon \R^{r+s}\to\spH$ can be seen as (semi-infinite) matrix with zero-mean random columns $\omega_1,\ldots,\omega_{r+s}$. The integer $r>0$ is the {\em target rank}, whereas $s>0$ is the {\em oversampling} parameter.
    \item[(2)] We construct $\OO{M} := (\OO{B}\OO{B}^{*})^p \OO{Y}$. Here $p \geq1$ is an integer corresponding to the number of power iterations performed.
    \item[(3)] {Basic algebraic manipulations yield that} 
\begin{equation}\label{eq:M_Sketch}
\OO{M} = \Creg^{\frac{1}{2}}\OO{\wsketch}\quad\text{ for }\quad\OO{\wsketch}:= \ECreg^{-1}(\ECxy\ECxy^*\ECreg^{-1})^p\OO{\sketch}.
\end{equation}
    \item[(4)] Hence, according to  Fact~\ref{fct:projector_onto_range}, we can express $\proj{\OO{M}} = \Creg^{\frac{1}{2}}\OO{\wsketch} (\OO{\wsketch}^* \ECreg \OO{\wsketch})^\dagger \OO{\wsketch}^* \Creg^{\frac{1}{2}}$, and obtain $\proj{\OO{M}}\OO{B} = \Creg^{\frac{1}{2}}\OO{\wsketch} (\OO{\wsketch}^* \ECreg \OO{\wsketch})^\dagger \OO{\wsketch}^* \ECxy$.
    \item[(6)] To compute the $r$-truncated SVD of $\proj{\OO{M}}\OO{B}$ we construct its left singular vectors by solving the eigenvalue equation for the operator $\proj{\OO{M}}\OO{B}\OO{B}^*\proj{\OO{M}}$. Using Lemma~\ref{lm:operator_EVD}, the left singular vectors are of the form $\Creg^{\frac{1}{2}}\OO{\wsketch} q_i$, where the vectors $q_i\in \R^{(r+s)}$ satisfy the matrix eigenvalue equation $F_0^\dagger F_1 F_0^\dagger F_0 q_{i} = \sigma^{2}_{i}q_{i}$ with
        \[
        F_0:= \OO{\wsketch}^* \ECreg \OO{\wsketch}\quad \text{ and } \quad F_1:= \OO{\wsketch}^* \ECxy\ECxy^* \OO{\wsketch}.
        \]
    \item[(7)] From Fact \ref{fct:GEP}, solutions of the matrix eigenvalue equation $F_0^\dagger F_1 F_0^\dagger F_0 q_{i} = \sigma^{2}_{i}q_{i}$ solve, as well, the following GEP 
        \begin{equation}\label{eq:rGEP}
        F_1 q_i = \sigma_i^2 F_0 q_i, \text{ with } q_i^* F_0 q_i = 1.
        \end{equation}
    Upon defining $Q_r := [q_1 \vert \ldots \vert q_r ]\in\R^{(r+s)\times r}$ as the matrix whose columns are the first $r$ leading eigenvectors of~\eqref{eq:rGEP}, Fact~\ref{fct:projector_onto_range} yields that $\proj{\SVDr{\proj{\OO{M}}\OO{B}}} = \Creg^{\frac{1}{2}}\ES^* \OO{\wsketch} Q_r Q_r^* \OO{\wsketch}^*\Creg^{\frac{1}{2}}$.
    \item[(8)] The $r$-truncated SVD of $\OO{B}$ is approximated by 
        \begin{equation}\label{eq:rSVD_final}
            \SVDr{\OO{B}} \approx \SVDr{\proj{\OO{M}}\OO{B}} = \proj{\SVDr{\proj{\OO{M}}\OO{B}}}\OO{B} = \Creg^{\frac{1}{2}}\OO{\wsketch} Q_{r}Q_{r}^{*}\OO{\wsketch}^{*}\ECxy.
        \end{equation} 
    \item[(9)] Finally, plugging~\eqref{eq:rSVD_final} into  $\EEstim_{r, \reg} = \SVDr{\OO{B}}^*\ECreg^{-\frac{1}{2}}$, we conclude that
        \begin{equation}\label{eq:r4_solution}
            \EEstim_{r, \reg}^{s, p} := \ECxy^*\OO{\wsketch} Q_{r}Q_{r}^{*}\OO{\wsketch}^{*}.
        \end{equation}
\end{itemize}

Notice how in the randomized reduced rank regression (R$^4$) procedure described above the main computational difficulty lies in obtaining $\OO{\wsketch}$ in step (3), while the rest of the algorithm can be efficiently performed in low-dimensional $\R^{(r+s)\times(r+s)}$ space. So, depending on the nature of the spaces $\spH$ and $\spG$, we will formalize two variants of the method given above:
\begin{itemize}
\item Primal R$^4$ algorithm, useful when $\min\{\dim(\spH), \dim(\spG)\} \leq n$,
\item Dual R$^4$ algorithm, useful when $n<\min\{\dim(\spH), \dim(\spG)\} \leq \infty$.
\end{itemize} 

\medskip
\subsection{Primal R$^4$ algorithm}
The algorithm in its primal formulation is just a formalization of the ideas of the previous paragraph. This variant is preferable when the number of data samples is larger than the (finite) dimension of at least one feature map $\fH$ or $\fG$. In Algorithm \ref{alg:r4_primal}, we detail every step of the algorithm in the case $\dim(\spH)\leq \dim(\spG)$, which can be assumed w.l.o.g. Since in this case 
\begin{equation}
\label{eq:TTstar}
\ECxy\ECxy^* = \ES^*\Ky\ES,
\end{equation}
we can always compute  $\ECxy\ECxy^*$, even when $\spG$ is infinite dimensional. Further, we write $\sketch$, $\wsketch$, $C_\reg$ and $S$ instead of $\OO{\sketch}$, $\OO{\wsketch}$, $\ECreg$ and $\ES$, respectively, emphasizing that these operators are matrices. Finally, to promote the numerical stability in the computation of projection in step (4), we follow the standard practice~\cite{Tropp2017} of orthonormalizing the basis of the approximate range through economy-size QR decomposition, hereby denoted as \texttt{qr\_econ}.

\begin{algorithm}[!ht]
\caption{Primal R$^4$:  $\dim(\spH) \leq n$}\label{alg:r4_primal}
\begin{algorithmic}[1]
\Require dataset $\Data$, power $p \geq 1$ , oversampling $s \geq 2$ and random sketching  $\sketch\!\in\!\R^{\dim(\spH)\times (r+s)}$
\State Compute matrix $N =  S^*\Ky S \in\R^{\dim(\spH)\times \dim(\spH)}$
\For{$j\in[p]$}
\State Solve $C_\reg\wsketch = \sketch$ 
\State Update $\sketch \gets N\wsketch$
\State Compute $\sketch\mathrm{, \_} = \mathrm{\texttt{qr\_econ}}(\sketch)$
\EndFor
\State Solve $C_\reg \wsketch = \sketch$
\State Compute $\quad F_0 = \wsketch^* \sketch \quad$  and  $\quad F_1 = \wsketch^* (N \wsketch)$
\State Compute the $r$ leading eigenpairs $(\sigma_i^2,q_i)$ of $F_1 q_i = \sigma_i^2 F_0 q_i$
\State Update $q_i\gets q_i / \sqrt{q_i^* F_0 q_i }$
\State Compute $\widehat{V}_{r,p} = \wsketch Q_r$, for $Q_r = [q_1\,\ldots\,q_r]\in\R^{(r+s)\times r}$
\Ensure $\EEstim_{r,\reg}^{s,p}:= \ECxy^* \widehat{V}_{r,p} \widehat{V}^{*}_{r,p}$
\end{algorithmic}
\end{algorithm}

\subsection{Dual R$^4$ algorithm}
To address the case $n<\dim(\spH)\leq \infty$, the main idea is to construct the sketching operator as $\OO{\sketch}:=\ES^*\sketch$, where $\sketch\in\R^{n\times (r+s)}$ is a Gaussian matrix. Before stating the proposed dual algorithm we report a Lemma that gives key insights on how to translate the operator regression problem into an equivalent matrix problem.
\begin{lemma}\label{lm:operator-to-matrix} 
Let $p\geq1$. Given $\sketch \in\R^{n\times k}$, let $\OO{\sketch}=\ES^*\sketch$. Then, recalling that $\OO{B} := \ECreg^{-\frac{1}{2}}\ECxy$, $\OO{Y} := \ECreg^{-\frac{1}{2}}\OO{\sketch}$ and \eqref{eq:M_Sketch}, for matrices $B := \Kx^{\frac{1}{2}} \Kreg^{-\frac{1}{2}}\Ky^{\frac{1}{2}}$, $Y := \Kreg^{-\frac{1}{2}}\Kx^{\frac{1}{2}}\sketch$, $\wsketch :=  \Kreg^{-1}\left(\Ky(\Id_{n} - \reg\Kreg^{-1})\right)^p \sketch$ and $M:=(BB^*)^p Y$ it holds that
\begin{equation}\label{eq:rkhs-to-matrix}
\OO{\wsketch}=\ES^*\wsketch,\quad \hnorm{\SVDr{\OO{B}}}^2 = \norm{\SVDr{B}}_{F}^2 \quad\text{ and }\quad  \hnorm{\proj{\OO{M}}\SVDr{\OO{B}}}^2 = \norm{ \proj{Y}\SVDr{B}}_{F}^2.
\end{equation}
\end{lemma}
\begin{proof}
Since $\OO{B} = \ECreg^{-\frac{1}{2}}\ECxy = \ES \Kreg^{-\frac{1}{2}}\EZ$, from Proposition \ref{prop:operator_SVD} we have that reduced SVD of $\OO{B}$ and $B$ can be expressed as $\SVDr{\OO{B}} = \ES U \Sigma V^* \EZ$ and $\SVDr{B} = \Kx^{\frac{1}{2}} U \Sigma V^* \Ky^{\frac{1}{2}}$, respectively. But then, clearly  $\hnorm{\SVDr{\OO{B}}}^2 = \norm{\SVDr{B}}^2_F=\sum_{i\in[r]}\sigma_i^2$. 

Next, since 
\[
M = (BB^*)^p Y= (\Kx^{\frac{1}{2}}\Kreg^{-\frac{1}{2}}\Ky\Kreg^{-\frac{1}{2}}\Kx^{\frac{1}{2}})^p \Kreg^{-\frac{1}{2}} \Kx^{\frac{1}{2}} \sketch  = \Kx^{\frac{1}{2}} \Kreg^{-\frac{1}{2}} (\Ky\Kx\Kreg^{-1})^p\sketch,
\]
recalling $ \wsketch =  \Kreg^{-1}\left(\Ky(\Id_{n} - \reg\Kreg^{-1})\right)^p \sketch = \Kreg^{-1} (\Ky\Kx\Kreg^{-1})^p\sketch$, we can write $M=  \Kreg^{\frac{1}{2}} \Kx^{\frac{1}{2}}  \wsketch$. In a similar way, we have that for the operator $\OO{M}$ holds 
\[
\OO{M} = (\OO{B}\OO{B}^*)^p \OO{Y}= (\ES^*\Kreg^{-\frac{1}{2}}\Ky\Kreg^{-\frac{1}{2}}\ES)^p \ECreg^{-\frac{1}{2}}\ES^* \sketch  = \ES^* \Kreg^{-\frac{1}{2}} (\Ky\Kx\Kreg^{-1})^p\sketch = \Creg^{\frac{1}{2}}\ES^*\wsketch,
\]
and, hence $\OO{\wsketch}=\ES^*\wsketch$. Now, let $V_r$ be $V$ truncated to its first $r$-columns. Then we have $\OO{P}_r  = \EZ^* V_r V_r^* \EZ$,  and hence $\SVDr{\OO{B}} = \ES^*\Kreg^{-\frac{1}{2}}\Ky V_rV_r^* \EZ$. Similarly $\SVDr{B} = B \Ky^{\frac{1}{2}} V_rV_r^* \Ky^{\frac{1}{2}} =  \Kx^{\frac{1}{2}}\Kreg^{-\frac{1}{2}}\Ky V_rV_r^* \Ky^{\frac{1}{2}}$. Therefore, it follows
\begin{eqnarray*}
\ES \proj{\OO{M}} \ES^* & =& \ES \OO{M}(\OO{M}^*\OO{M})^{\dagger}\OO{M}^* \ES^* \\
&=& \ES \ECreg^{\frac{1}{2}} \ES^* \wsketch (\wsketch^* \ES \ECreg \ES^* \wsketch)^\dagger  \wsketch^* \ES \ECreg^{\frac{1}{2}} \ES^* \\
&  = & \Kx \Kreg^{\frac{1}{2}} \wsketch (\wsketch^* \Kx \Kreg \wsketch)^\dagger  \wsketch^*  \Kreg^{\frac{1}{2}} \Kx \\&=&  \Kx^{\frac{1}{2}} M (M^* M)^\dagger  M^*  \Kx^{\frac{1}{2}}  \\
&
=& \Kx^{\frac{1}{2}} \proj{M} \Kx^{\frac{1}{2}},
\end{eqnarray*}
and, consequently,
\begin{eqnarray*}
\hnorm{\proj{\OO{M}}\SVDr{\OO{B}}}^2 & = &\tr\left(\SVDr{\OO{B}}^*\proj{\OO{M}}^2\SVDr{\OO{B}} \right) \\&=&  \tr\left( \EZ^* V_r V_r^*\Ky \Kreg^{-\frac{1}{2}} \ES \proj{\OO{M}} \ES^*\Kreg^{-\frac{1}{2}}\Ky V_rV_r^* \EZ \right)  \\
& =& \tr\left( \EZ \EZ^* V_r V_r^*\Ky^{\frac{1}{2}} B^* \proj{M} B\Ky^{\frac{1}{2}} V_rV_r^* \right)  \\
&=& \tr\left( \Ky V_r V_r^*\Ky^{\frac{1}{2}} B^* \proj{M} B\Ky^{\frac{1}{2}} V_rV_r^* \right) \\
& = & \tr\left( \Ky^{\frac{1}{2}} V_r V_r^*\Ky^{\frac{1}{2}} B^* \proj{M}^2 B\Ky^{\frac{1}{2}} V_rV_r^* \Ky^{\frac{1}{2}} \right) \\
&=& \norm{\proj{M} \SVDr{B}}^2_F.
\end{eqnarray*}
concluding the proof.
\end{proof}
Now, recalling that $\EEstim_{r,\reg}^{s,p}:\ECxy^*\OO{\wsketch} Q_{r}Q_{r}^{*}\OO{\wsketch}^{*} = \EZ^*\ES\ES^*\wsketch Q_{r}Q_{r}^{*}\wsketch^{*}\ES = \EZ^*\Kx\wsketch Q_{r}Q_{r}^{*}\wsketch^{*}\ES$, the dual R$^4$, given in Algorithm \ref{alg:r4_dual}, readily follows.
\begin{algorithm}[t!]
\caption{Dual R$^4$ : $n<\dim(\spH) \wedge \dim(\spG) \leq \infty$}\label{alg:r4_dual}
\begin{algorithmic}[1]
\Require dataset $\Data$, power $p \geq 1$ , oversampling $s \geq 2$ and random sketching  $\sketch\in\R^{n\times (r+s)}$
\For{$j\in[p]$}
\State Solve $\Kreg \wsketch = \sketch$
\State Update $\sketch \gets \Ky(\sketch -\reg \wsketch)$
\State Compute $\sketch\mathrm{, \_} = \mathrm{\texttt{qr\_econ}}(\sketch)$
\EndFor
\State Solve $\Kreg \wsketch = \sketch$
\State Compute $\quad F_0 = \wsketch^* \Kx \sketch \quad$  
\State Update $\sketch \gets \Ky(\sketch -\reg \wsketch)$
\State Compute  $\quad F_1 = \wsketch^* \Kx \sketch$
\State Compute the $r$ leading eigenpairs $(\sigma_i^2,q_i)$ of $F_1 q_i = \sigma_i^2 F_0 q_i$
\State Update $q_i\gets q_i / \sqrt{q_i^* F_0 q_i }$
\State Compute $\widehat{V}_{r,p} = \wsketch Q_r$, for $Q_r = [q_1\,\ldots\,q_r]\in\R^{(r+s)\times r}$
\State Compute $\widehat{U}_{r,p} = \Kx \widehat{V}_{r,p}$ 
\Ensure $\EEstim_{r,\reg}^{s,p}:= \EZ^* \widehat{U}_{r,p} \widehat{V}_{r,p}^* \ES $
\end{algorithmic}
\end{algorithm}

We conclude this section with a remark on computational complexity of R$^4$ implementations. Recalling Proposition~\ref{prop:RRR_comp}, one sees that the computational cost of R$^{4}$ is fixed a-priori by the choice of $s$ and $p$, while solving GEP in \eqref{eq:GEP_primal} depends on the convergence rate of the generalized eigenvalue solver of choice. 

In the following section, we show that the regularized empirical risk of R$^4$ estimator exhibits fast convergence to the optimal value as the parameters $s$ and $p$ increase.

\section{Error bounds}\label{sec:bounds}
In this section we provide bounds for $\EE[\ERisk(\EEstim_{r, \reg}^{s,p})] - \ERisk(\EEstim_{r, \reg})$. Here, the expectation value is intended over the Gaussian matrix $\sketch$ appearing in Algorithms \ref{alg:r4_primal} and \ref{alg:r4_dual}. The main conclusion is that $\EE[\ERisk(\EEstim_{r, \reg}^{s,p})] - \ERisk(\EEstim_{r, \reg})$ can be made arbitrarily small by increasing the oversampling $s\geq2$ and powering $p\geq1$ parameters. Moreover, when the singular values of $B$ have a rapid decay, the bound is tight already for small values of $s$ and $p$, e.g, $s\in[r]$ and $p\in[3]$. 

The proofs rely on the following proposition linking the randomized SVD of operators to recently developed error analysis for randomized low-rank approximation methods for matrices in \cite{Diouane2022}, which we restate in notation adapted to our manuscript.

\begin{proposition}\label{prop:main}
Given  $p\geq1$, $s\geq2$, data $\mathcal{D}:=\{(x_i,y_i)\}_{i\in[n]}$ and operator  $\OO{\sketch} \colon \R^{r+s}\to\spH$, for $\EEstim_{r, \reg}^{s,p}$ given in \ref{eq:r4_solution} it holds that 
\begin{equation}\label{eq:error_rkhs}
\ERisk(\EEstim_{r,\reg}^{s,p}) - \ERisk(\EEstim_{r,\reg})\leq \hnorm{(\Idop_{\spH} - \proj{\OO{M}})\SVDr{\OO{B}}}^2. %
\end{equation}

\end{proposition}
\begin{proof}
Let us start by recalling that due to \eqref{eq:RRR_solution}, the minimizer of \eqref{eq:RRR} is $\EEstim_{r, \reg} = \SVDr{\OO{B}}^*\ECreg^{-\frac{1}{2}}$ while according to \eqref{eq:rSVD_final} its randomized approximation is
$\EEstim_{r,\reg}^{s,p} = (\proj{\SVDr{\proj{\OO{M}}\OO{B}}}\OO{B}  )^*  \ECreg^{-\frac{1}{2}}$. Therefore, \eqref{eq:RRR_risk} implies
\begin{align*}
\ERisk(\EEstim_{r,\reg}^{s,p}) - \ERisk(\EEstim_{r,\reg}) & = \hnorm{\EEstim_{r,\reg}^{s,p}\ECx_\reg^{\frac{1}{2}}  - \ECyx\ECx_\reg^{-\frac{1}{2}}}^2 - \hnorm{\EEstim_{r,\reg}\ECx_\reg^{\frac{1}{2}}  - \ECyx\ECx_\reg^{-\frac{1}{2}}}^2 \\
& =  \hnorm{(\proj{\SVDr{\proj{\OO{M}}\OO{B}}}\OO{B})^*  - \OO{B}^*}^2 - \hnorm{\SVDr{\OO{B}^*}  - \OO{B}^*}^2 \\
& = \hnorm{(\Idop_{\spH} - \proj{\SVDr{\proj{\OO{M}}\OO{B}}})\OO{B}}^2 - \hnorm{\SVDr{\OO{B}\}}  - \OO{B}}^2 \\
& = \hnorm{\OO{B}}^2 - \hnorm{\proj{\SVDr{\proj{\OO{M}}\OO{B}}}\OO{B}}^2 - \hnorm{\OO{B}}^2 + \hnorm{\SVDr{\OO{B}}}^2  \\
& =  \hnorm{\SVDr{\OO{B}}}^2 - \hnorm{\proj{\SVDr{\proj{\OO{M}}\OO{B}}}\OO{B}}^2\\
& = \hnorm{(\Idop_{\spH} - \proj{\OO{M}})\SVDr{\OO{B}}}^2 + \hnorm{\proj{\OO{M}} \SVDr{\OO{B}}}^2 - \hnorm{\proj{\SVDr{\proj{\OO{M}}\OO{B}}}\OO{B}}^2.
\end{align*}
We now show the upper-bound $\hnorm{\proj{\OO{M}} \SVDr{\OO{B}}}^2 \leq \hnorm{\proj{\SVDr{\proj{\OO{M}}\OO{B}}}\OO{B}}^2$. Indeed, observe that by \eqref{eq:rSVD_final} it holds $\proj{\SVDr{\proj{\OO{M}}\OO{B}}}\OO{B} = \SVDr{\proj{\OO{M}}\OO{B}}$. So, denoting by $\OO{P}_r$ the projector on the subspace of $r$ leading right singular vectors of $\OO{B}$, we have that
\begin{align*}
\hnorm{\proj{\OO{M}} \SVDr{\OO{B}}}^2 - \hnorm{\proj{\SVDr{\proj{\OO{M}}\OO{B}}\OO{B}}}^2 & = \hnorm{\proj{\OO{M}}\SVDr{\OO{B}}}^2 - \hnorm{\SVDr{\proj{\OO{M}}\OO{B}}}^2 \\
& = \hnorm{\proj{\OO{M}} \OO{B} \OO{P}_r}^2 - \hnorm{\SVDr{\proj{\OO{M}}\OO{B}}}^2 \\
& = \hnorm{\proj{\OO{M}}\OO{B}}^2 - \hnorm{\proj{\OO{M}}\OO{B} (\Idop_\spG - \OO{P}_r)}^2 - \hnorm{\SVDr{\proj{\OO{M}}\OO{B}}}^2\\
& = \hnorm{\proj{\OO{M}}\OO{B}-\SVDr{\proj{\OO{M}}\OO{B}} }^2 - \hnorm{\proj{\OO{M}} (\OO{B} -\SVDr{\OO{B}})}^2.
\end{align*}
But since $\rank(\proj{\OO{M}} \SVDr{\OO{B}})\leq \rank(\SVDr{\OO{B}})\leq r$, using the  Eckart–Young–Mirsky theorem we have that
\begin{eqnarray*}
    \hnorm{\proj{\OO{M}} \OO{B} - \SVDr{\proj{\OO{M}}\OO{B}} }^2 &=& \min_{\widetilde{\OO{B}} \in \HSr}\hnorm{\proj{\OO{M}} \OO{B}  - \widetilde{\OO{B}} }^2 \\
    &\leq& \hnorm{\proj{\OO{M}} (\OO{B}  -\SVDr{\OO{B}})}^2,
\end{eqnarray*}
which concludes the proof.
\end{proof}

\begin{theorem}[Theorem 3.17 and Remark 3.18 of~\cite{Diouane2022}]\label{thm:aux_bound}
Given $B\in\R^{m \times n}$ such that $m\geq n$, let $B=U\Sigma V^*$ be its full SVD, $\SVDr{B} = U_r \Sigma_r V_r^*$ its $r-$truncated SVD, for some $r\in[n]$ and let $U = [U_r \; \underline{U}_r]$. Let $M\in\R^{n \times (r+s)}$, $s\geq2$, be a Gaussian matrix such that $M\sim\Nc(0,G)$, where $G\in\R^{n\times n}$ is a covariance matrix. If $G_r := U_r^* G U_r$ is nonsingular, then 
\begin{equation}\label{eq:aux_bound}
\EE \norm{[I-\proj{M}] \SVDr{B}}^2_F \leq \min\left\{\frac{r a_r}{r+a_r} \norm{B}^2,\; b_r \right\},
\end{equation}
where 
\begin{align}
a_r & := \tr(G_{\perp,r}G_r^{-2} G_{\perp,r}^* ) + \frac{\tr(\underline{G}_r - G_{\perp,r} G_r^{-1} G_{\perp,r}^*) \tr (G_r^{-1})}{s-1}, \label{eq:aux_bound_a}\\
b_r & := \tr(G_{\perp,r}G_r^{-1} \Sigma_r^2 G_r^{-1} G_{\perp,r}^* ) + \frac{\tr(\underline{G}_r - G_{\perp,r} G_r^{-1} G_{\perp,r}^*) \tr ( \Sigma_r G_r^{-1}\Sigma_r)}{s-1},  \label{eq:aux_bound_b}
\end{align}
$G_{\perp,r}:= \underline{U}_r^* G U_r$ and $\underline{G}_r:= \underline{U}_r^* G \underline{U}_r$.
\end{theorem}

Now, collecting Lemma \ref{lm:operator-to-matrix}, Proposition \ref{prop:main} and Theorem \ref{thm:aux_bound},  we immediately obtain the following expected error bound for the randomized reduced rank estimator \ref{eq:r4_solution}.

\begin{theorem}\label{thm:main_1}
Given data $\mathcal{D}:=\{(x_i,y_i)\}_{i\in[n]}$, $p\geq1$ and $s\geq2$, let $r\leq\rank(\ECxy)$ and $\Omega\in\R^{n\times (r+s)}$ be with i.i.d. columns from $\Nc(0,\Ky)$. If $\OO{\sketch}=\ES^*\sketch$,  then, for the randomized reduced rank estimator $\EEstim_{r, \reg}^{s,p}$ given in \ref{eq:r4_solution} it holds that 
\begin{equation}\label{eq:main_1}
\EE[\ERisk(\EEstim_{r, \reg}^{s,p})]-\ERisk(\EEstim_{r, \reg}) \leq \min\left\{\frac{r a_r}{r+a_r} \sigma_1^2,\; b_r \right\},
\end{equation}
where
\begin{align}
a_r & := \frac{1}{s-1} \Bigg[\sum_{i=r+1}^{n}\Big(\frac{\sigma_i}{\sigma_{r+1}}\Big)^{4p+2} \Bigg]\Bigg[ \sum_{i=1}^{r}\Big(\frac{\sigma_{r+1}}{\sigma_i}\Big)^{4p+2}\Bigg], \label{eq:main_1_a}\\
b_r & := \frac{\sigma_{r+1}^2}{s-1} \Bigg[\sum_{i=r+1}^{n}\Big(\frac{\sigma_i}{\sigma_{r+1}}\Big)^{4p+2}\Bigg]\Bigg[  \sum_{i=1}^{r}\Big(\frac{\sigma_{r+1}}{\sigma_i}\Big)^{4p}\Bigg],  \label{eq:main_1_b}
\end{align}
and $\sigma_1\geq\cdots\geq\sigma_n$ are singular values of $\OO{B}=\ECx_\reg^{-\frac{1}{2}}\ECxy$.
\end{theorem}
\begin{proof}
First, observe that Proposition \ref{prop:main} guarantees that almost surly 
\[
\ERisk(\EEstim_{r,\reg}^{s,p}) - \ERisk(\EEstim_{r,\reg})\leq \hnorm{(\Idop_{\spH} - \proj{\OO{M}})\SVDr{\OO{B}}}^2 = \hnorm{\SVDr{\OO{B}}}^2 - \hnorm{\proj{\OO{M}}\SVDr{\OO{B}}}^2.
\]
Hence, due to Lemma \ref{lm:operator-to-matrix}, 
\begin{equation}\label{eq:risk_operator_to_matrix}
\ERisk(\EEstim_{r,\reg}^{s,p}) - \ERisk(\EEstim_{r,\reg}) \leq \norm{\SVDr{B}}^2_F -  \norm{\proj{M} \SVDr{B}}^2_F = \norm{(I_n - \proj{M})\SVDr{B}}_{F}^2.
\end{equation} 
Next, since $L = \EE\left[\sketch\sketch^*\right]$, the covariance of the random matrix $M$ is
\begin{equation}\label{eq:gauss_covariance_1}
G = \EE \left[ MM^*\right] = (BB^*)^p \Kx^{\frac{1}{2}} \Kreg^{-\frac{1}{2}}\EE\left[\sketch\sketch^*\right] \Kreg^{-\frac{1}{2}} \Kx^{\frac{1}{2}}(BB^*)^p = (BB^*)^{2p+1}.
\end{equation}
Hence, the projected covariance becomes $G_r = (\Kx^{\frac{1}{2}} U_r)^*(BB^*)^{2p+1} (\Kx^{\frac{1}{2}} U_r)  = \Sigma_r^{4p+2}$. Since, $r\leq \rank(\ECxy^*\ECxy)=\rank(\OO{B})$, $G_r$ is nonsingular, and  \eqref{eq:main_1} holds with $a_r$ and $b_r$ computed according to \eqref{eq:aux_bound_a} and \eqref{eq:aux_bound_b}, respectively. Using that $G_{\perp,r} = 0$ and  $\underline{G}_r = \underline{\Sigma}_r^{4p+2}$, where $\underline{\Sigma}_r = {\rm diag}(\sigma_{r+1},\ldots,\sigma_n)$,  direct computation yields \eqref{eq:main_1_a} and \eqref{eq:main_1_b}.
\end{proof}

While the previous theorem provides the bound that quickly becomes zero by increasing $s$ and $p$, the main downside lies in the fact that generating anisotropic Gaussian sketch $\Omega\sim\Nc(0,\Ky)$ in certain cases can be computationally expensive. Namely, we have the following:

\begin{corollary}\label{cor:strong}
Given data $\mathcal{D}:=\{(x_i,y_i)\}_{i\in[n]}$, $p\geq1$ and $s\geq2$, let $r\leq\rank(\ECxy)$. Then, the following holds
\begin{enumerate}
\item Let $\dim(\spG)<\infty$. If $\tilde{\omega}_i\in\R^{\dim(\spG)}$, $i\in[r+s]$, are i.i.d. from $\Nc(0,I_{\dim(\spH)})$ and $\sketch = [T \tilde{\omega}_1\,\vert\,\ldots\,\vert\, T \tilde{\omega}_{r+s}]$, then primal R$^4$ Algorithm \ref{alg:r4_primal} produces $\EEstim_{r, \reg}^{s,p}$ so that \eqref{eq:main_1} holds true. 
\item Let $\dim(\spG)=\infty$, and let $N = R \Lambda R^*$ be eigenvalue decomposition of matrix $N = \ECxy\ECxy^* \in\R^{\dim(\spH)\times \dim(\spH)}$. If $\tilde{\omega}_i\in\R^{\dim(\spG)}$, $i\in[r+s]$, are i.i.d. from $\Nc(0,I_{\dim(\spH)})$ and $\sketch = [R \Lambda^{\frac{1}{2}} \tilde{\omega}_1\,\vert\,\ldots\,\vert\, R \Lambda^{\frac{1}{2}} \tilde{\omega}_{r+s}]$, then primal R$^4$ Algorithm \ref{alg:r4_primal} produces $\EEstim_{r, \reg}^{s,p}$ so that \eqref{eq:main_1} holds true. 
\item Let $\Ky = R^* R$ be Cholesky factorization. If $\tilde{\omega}_i\in\R^{n}$, $i\in[r+s]$, are i.i.d. from $\Nc(0,I_n)$ and $\sketch = [R^* \tilde{\omega}_1\,\vert\,\ldots\,\vert\, R^* \tilde{\omega}_{r+s}]$, then dual R$^4$ Algorithm \ref{alg:r4_dual} produces $\EEstim_{r, \reg}^{s,p}$ so that \eqref{eq:main_1} holds true.
\end{enumerate}
\end{corollary}

Therefore, if we have $\dim(\spH)<<n$, we can efficiently apply primal R$^4$ algorithm with appropriate sketching to obtain strong guarantees of Theorem \ref{thm:main_1}. Otherwise, implementing the sketch in Corollary \ref{cor:strong} becomes too expensive. So, in the following result, we prove that dual R$^4$ algorithm can provably attain arbitrary small errors even using isotropic Gaussian sketch $\Omega\sim\Nc(0,\Id_n)$.

\begin{theorem}\label{thm:main_2}
Given data $\mathcal{D}:=\{(x_i,y_i)\}_{i\in[n]}$, if $r\leq\rank(\Kx\Ky)$, $p\geq1$, $s\geq2$ and $\Omega\in\R^{n\times (r+s)}$ with iid columns from $\Nc(0,\Id_n)$, than for the output $\EEstim_{r, \reg}^{s,p}$ of the dual R$^4$ Algorithm \ref{alg:r4_dual} it holds that 
\begin{equation}\label{eq:main_2}
\EE[\ERisk(\EEstim_{r, \reg}^{s,p})]-\ERisk(\EEstim_{r, \reg}) \leq \min\left\{\frac{r a_r}{r+a_r} \sigma_1^2,\; b_r \right\},
\end{equation}
where 
\begin{align}
a_r & := \frac{\norm{\Ky}}{\sigma_r^2} \Bigg[\sum_{i=r+1}^{n}\Big(\frac{\sigma_{i}}{\sigma_r}\Big)^{4p}\Bigg] \Bigg[ 1+  \frac{1}{s-1} \sum_{i=1}^{r}\Big(\frac{\sigma_r}{\sigma_{i}}\Big)^{4p+2}\Bigg], \label{eq:main_2_a}\\
b_r & := \norm{\Ky} \Bigg[\sum_{i=r+1}^{n}\Big(\frac{\sigma_{i}}{\sigma_r}\Big)^{4p}\Bigg] \Bigg[ \frac{\sigma_1^2}{\sigma_r^2}+  \frac{1}{s-1} \sum_{i=1}^{r}\Big(\frac{\sigma_r}{\sigma_{i}}\Big)^{4p}\Bigg],  \label{eq:main_2_b}
\end{align}
and $\sigma_1\geq\cdots\geq\sigma_n$ are singular values of $\OO{B}=\ECx_\reg^{-\frac{1}{2}}\ECxy$.
\end{theorem}
\begin{proof}
Noting that $\rank(\Kx\Ky) = \rank(\OO{B})$, the only difference from the proof of Theorem \ref{thm:main_1} is in the derivation of the bound from \eqref{eq:risk_operator_to_matrix}. The difficulty in this case lies in the fact that since 
\begin{equation}\label{eq:gauss_covariance_2}
G = \EE \left[ MM^*\right] = (BB^*)^p \Kx \Kreg (BB^*)^p,
\end{equation}
in general, the projected covariance $G_{\perp,r} \neq 0$. Hence, it is not evident that the first terms in \eqref{eq:aux_bound_a} and \eqref{eq:aux_bound_b} can be made arbitrarily small. To address this, observe that $\Id_n\succeq \Id_n - \reg\Kreg^{-1} = \Kx\Kreg^{-1}$ and $\Ky\preceq \norm{\Ky}I_n$ implies 
\begin{equation}\label{eq:covarinace_bounds}
\norm{L}^{-1} (BB^*)^{2p+1} \preceq G \preceq (BB^*)^{2p}.
\end{equation}
Moreover, the projected covariance is positive definite, since 
\begin{eqnarray}
G_r = (\Kx^{\frac{1}{2}} U_r)^*G (\Kx^{\frac{1}{2}} U_r) \succeq \norm{\Ky}^{-1}(\Kx^{\frac{1}{2}} U_r)^*(BB^*)^{2p+1} (\Kx^{\frac{1}{2}} U_r)\succeq \norm{\Ky}^{-1} \Sigma_r^{4p+2} \succ 0,
\end{eqnarray}
which implies that its Schur complement w.r.t. to the first $r$ indices $ \underline{G}_r - G_{\perp,r} G_r^{-1} G_{\perp,r}^*$ is positive semi-definite. Hence, we obtain $G_{\perp,r} G_r^{-1} G_{\perp,r}^*\preceq \underline{G}_r$. However, projecting \eqref{eq:covarinace_bounds} yields  $\underline{G}_r \preceq \underline{\Sigma}_r^{4p}$, which implies that  $G_r^{-1} \preceq \norm{\Ky} \Sigma_r^{-4p-2}$ and 
\begin{equation}\label{eq:covarinace_bounds_schur}
G_{\perp,r} G_r^{-2} G_{\perp,r}^* \preceq \norm{G_r^{-1}} G_{\perp,r} G_r^{-1} G_{\perp,r}^*  \preceq \norm{\Ky} \norm{\Sigma_r^{-4p-2}}  \underline{G}_r \preceq \frac{\norm{\Ky}}{\sigma_r^{4p+2}} \underline{\Sigma}_r^{4p}.
\end{equation}
Taking the trace we conclude that 
\begin{equation}\label{eq:a_first_term}
\tr\left( G_{\perp,r} G_r^{-2} G_{\perp,r}^* \right) \leq \frac{\norm{\Ky}}{\sigma_r^2}\sum_{i=r+1}^{n} \Big(\frac{\sigma_i}{\sigma_r}\Big)^{4p}
\end{equation}
In a similar manner, we obtain 
\begin{equation}\label{eq:b_first_term}
\tr\left( G_{\perp,r} G_r^{-1} \Sigma_r^2 G_r^{-1}G_{\perp,r}^* \right) \leq \norm{\Ky} \frac{\sigma_1^2}{\sigma_r^2}\sum_{i=r+1}^{n} \Big(\frac{\sigma_i}{\sigma_r}\Big)^{4p},
\end{equation}
\begin{equation}\label{eq:second_terms}
\tr\left( G_r^{-1} \right) \leq \norm{\Ky} \sum_{i=1}^{r} \frac{1}{\sigma_i^{4p+2}}\quad\text{ and }\quad \tr\left( \Sigma_r G_r^{-1} \Sigma_r \right) \leq \norm{\Ky} \sum_{i=1}^{r} \frac{1}{\sigma_i^{4p}},
\end{equation}
and \eqref{eq:main_2_a} and \eqref{eq:main_2_b} follow using that $\tr\big( \underline{G}_r - G_{\perp,r} G_r^{-1} G_{\perp,r}^* \big) \leq \tr\left( \underline{\Sigma}_r^{4p} \right) = \sum_{i=r+1}^{n}\sigma_i^{4p}$.  
\end{proof}

As a final remark of this section note that $\sigma_{r+1}<\sigma_r$ suffices to assure that $a_r$ and $b_r$ from \eqref{eq:main_2_a} and \eqref{eq:main_2_b}, respectively, converge to zero with the linear rate  $(\sigma_{r+1} / \sigma_r)^4$  as $p$ grows.

\section{Numerical experiments} \label{sec:experiments}
The numerical evaluation of our randomized algorithms is comprised of three experiments: a simple multivariate linear regression to benchmark the bounds of Theorems~\ref{thm:main_1} and \ref{thm:main_2}, a large-scale vector-valued regression problem from computational neuroscience and a Koopman operator regression for a one-dimensional dynamical system, i.e., the noisy logistic map. In these large-scale problems, we are mainly interested in probing the enhanced performance of {\em randomized} reduced rank regression.

\subsection*{Linear  system}
We define a multivariate linear regression problem $Y = AX + \xi$ with $X = \{ x_1, \ldots, x_n \}$,  $x_{i} \in \R^{d}$, $x_i \sim \mathcal{N}(0, \Id)$, $A\in \R^{d\times d}$ and $\xi \sim \mathcal{N}(0, \Diag(\texttt{std}))$ standard Gaussian random vectors. We designed a low-rank matrix $A = U\Sigma U^{\top}$ with $U$ being a random orthogonal matrix $\in \R^{d\times d}$ and 
$\Sigma :=\Diag(\sigma_1, \ldots, \sigma_d) \in \R^{d\times d}$. The singular values are defined as
$\sigma_i := 1/(1+exp(-z_i))$, where $z = r-(1, \ldots, d)/\tau$, with $r$ being the number of singular values close to 1 and $\tau$ being the decay length representing the number of singular values lying in the slope part of the sigmoid curve approaching 0.

For our numerical experiments, we set $d = 100$, $r=10$, $\tau=5$, $\texttt{std}=0.1$,  the number of training and test points $n = 1000$ and let $\fG(x) = \fH(x) = x$ for all $x \in \X$. We trained reduced rank estimators~\eqref{eq:RRR_solution} both with the method as implemented in~\cite{Kostic2022}, based on Arnoldi iterations (R$^3$), and the method described in the present work, based on randomized SVD techniques (R$^4$). Each estimator was trained by fixing the regularization parameter $\reg = 10^{-6}$ and the number of powering iterations $p=1$. 

{\small
\begin{figure}[t!]
\begin{center}
\includegraphics[scale=1]{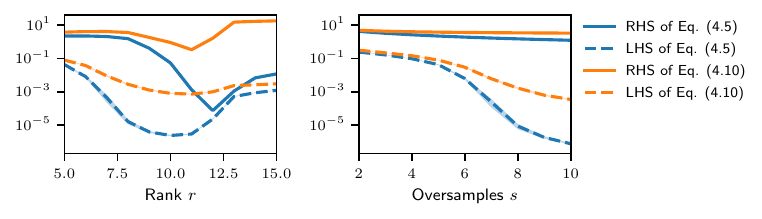}
    \caption{Comparison between the upper bounds (solid lines) derived in Theorems~\ref{thm:main_1} and \ref{thm:main_2}  and the estimated error (dashed lines) as a function of rank (left, $s=5$) and of oversampling parameter $s$ (right, $r=5$). The estimated error is represented by the average ($\pm$ 95\% confidence interval) over 1000 independent seed initializations.}\label{fig:noisylinear}
    \end{center}
\end{figure}
}

In Fig.~\ref{fig:noisylinear} we compare the actual value of $\EE[\ERisk(\EEstim_{r, \reg}^{s,p})]-\ERisk(\EEstim_{r, \reg})$ against the upper bounds derived in Theorems~\ref{thm:main_1} and \ref{thm:main_2} and we can appreciate how the empirical quantities are well below the theoretical bounds. The left panel shows the comparison conducted by varying the rank while keeping the oversampling parameter $s$ fixed at 5. Conversely, the right panel shows the influence of the oversampling parameter $s$ with rank set to 5.

Moreover, in Fig. \ref{fig:noisylinear_timings} we compare risk and fit time of R$^3$ and R$^4$ algorithms as a function of the training sample size $n$. The estimators were trained with the parameters $r=15$, $\reg = 10^{-6}$ and, concerning R$^4$, $p=1$ and $s=20$. We observe that the risk difference between R$^4$ and R$^3$ estimators on training data decreases as $n$ increases and for test data is virtually 0 (left and center panels, respectively). This empirically suggests that the inaccuracies introduced by the randomized procedure do not affect the generalization performances of the estimator. Further, our na\^ive Python implementation attains, on average, a $\sim 8.6\times$ faster computation speed compared w.r.t. the highly optimized Arnoldi iterations implemented in the \texttt{SciPy}~\cite{Scipy2020} package (right panel).

{\small
\begin{figure}[t!]
\centering
\includegraphics[scale=1]
{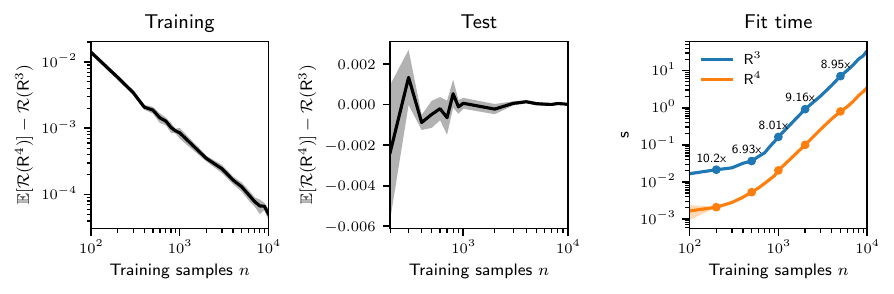}
    \caption{Risk difference evaluated on the training (left) and test (center) sets, and comparison of fit time between R$^4$ and R$^3$ (right). The estimated quantities are represented by the average over 10 independent seed initializations.} \label{fig:noisylinear_timings}
\end{figure}
}

\subsection*{Large scale neuroscience dataset} 
To further prove the effectiveness of R$^4$, we challenge it on a large scale neuroscience task: the Natural Scenes Dataset (NSD) \cite{Allen2021, Gifford2023}. This dataset comprises high-resolution fMRI recordings obtained from 8 healthy subjects, each exposed to a set of 10000 images. Among these images, drawn from the COCO dataset \cite{Lin2014}, 9000 were unique stimuli and 1000 were shared across subjects. This provided us with a total of 73000 unique recording/stimulus pairs. The goal, here, is to assess the ability of our algorithm to predict the neural activity evoked by intricate natural visual stimuli.

We face this problem by learning separate reduced rank estimators $\EEstim_{r, \reg}$ for each subject and hemisphere. Each estimator acts as a mapping between the two high-dimensional input (images) space, where each image was described by several features (i.e., image embedding), and the output (fMRI) space, described by 19004 and 20544 vertices, respectively, for left and right hemisphere (except for subjects 6 and 8 which have 20220 and 20530 vertices respectively for the right hemisphere). To obtain a compact representation of the visual stimuli, we derive image embeddings from the last convolutional layer (\texttt{features.12}) of the well-known AlexNet model \cite{Krizhevsky2017} (9216-dimensional feature vector), and from the recent computer vision foundation model DINOv2 (ViT-g) \cite{Oquab2023}. Concerning DINOv2, we follow the approach outlined in \cite{Caron2021}, which consists in concatenating the output of the \texttt{[CLS]} token and the GeM-pooled \cite{Radenovic2019} output patch tokens from the last 20 blocks. This yields a 61440-dimensional feature vector for DINOv2, which, when combined with the one obtained from AlexNet, produces a 70565-dimensional image embedding.

In this experiment, we trained the estimators R$^3$ and R$^4$ with two different kernel combinations: (i) a linear kernel applied to both input and output spaces, and (ii) a Mat\'{e}rn kernel (with smoothness parameter $\nu=0.5$) for the input space and a linear kernel for the output space. The estimators' performance was evaluated and compared in terms of mean Pearson's correlation ($r$), mean squared correlation $r^2$, and mean noise-normalized squared correlation (NNSC, i.e., the challenge's metric described in \cite{Gifford2023}), computed over all the vertices from all subjects.

The hyperparameters of each estimator were tuned through a grid-search procedure using an independent validation set. The tuned hyperparameters were:
\vspace{.12truecm}
\begin{itemize}
    \item the regularization parameter $\gamma \in [10^{-5}, 10^{-2}]$,
    \item the length-scale of Mat\'{e}rn kernel, selected from a set of quantiles (0.05, 0.25, 0.5, 0.75, 0.95) derived from the pairwise distances among 1000 randomly selected image embeddings within the training set.
\end{itemize}
\vspace{.12truecm}
\begin{table}[t!]
\begin{center}
\footnotesize
\caption{Comparison between baseline, R$^3$ and R$^4$ models in terms of Pearson's correlation ($r$), squared correlation ($r^2$) and noise-normalized squared correlation (NNSC). Models are tested with private and challenge test sets. For R$^4$ models, hyperparameter optimization (HPO) is based on results obtained from a validation set of either R$^3$ (HPO: R$^3$) or R$^4$ (HPO: R$^4$). Best results are in bold.}
\label{tab:neuro}
\begin{tabular}{cccc|ccc|c}
\multicolumn{1}{c}{} & \multicolumn{1}{l}{} & \multicolumn{1}{l}{} & \multicolumn{1}{c}{}  & \multicolumn{3}{|c|}{Private test set} & Challenge test set \\
\hline
\multicolumn{1}{c}{Model} & $K_X$ & $K_Y$ & HPO & \multicolumn{1}{|c}{$r$} & $r^2$ & NNSC & \multicolumn{1}{|c}{NNSC} \\
\hline
\multicolumn{1}{c}{Baseline} & / & / & / & \multicolumn{1}{c}{0.3923} & 0.1816 & 55.172 & \multicolumn{1}{c}{47.639} \\[0.2cm]

\multicolumn{1}{c}{R$^3$} & Linear & Linear & R$^3$ & \multicolumn{1}{c}{0.4105} & 0.1993& 59.537& \multicolumn{1}{c}{51.749} \\
\multicolumn{1}{c}{R$^4$} & Linear & Linear & R$^3$ &  \multicolumn{1}{c}{0.4101} & 0.1988& 59.443& \multicolumn{1}{c}{51.637} \\
\multicolumn{1}{c}{R$^4$} & Linear & Linear & R$^4$ & \multicolumn{1}{c}{0.4101} & 0.1989& 59.424& \multicolumn{1}{c}{51.634} \\[0.2cm]
\multicolumn{1}{c}{R$^3$} & Mat\'{e}rn & Linear & R$^3$ & \multicolumn{1}{c}{\textbf{0.4147}} & \textbf{0.2015} & \textbf{61.182} & \multicolumn{1}{c}{\textbf{52.951}} \\
\multicolumn{1}{c}{R$^4$} & Mat\'{e}rn & Linear & R$^3$ & \multicolumn{1}{c}{0.4142} & 0.2010 & 61.024 & \multicolumn{1}{c}{52.828} \\
\multicolumn{1}{c}{R$^4$} & Mat\'{e}rn & Linear & R$^4$ & \multicolumn{1}{c}{0.4142} & 0.2011 & 61.023 & \multicolumn{1}{c}{52.843} \\
\end{tabular}
\end{center}
\end{table}

In Table \ref{tab:neuro}, we show the results of the reduced rank estimators and a baseline model. As in \cite{Gifford2023}, the baseline model consists of (i) a dimensionality reduction phase of the image embeddings to 100 dimensions through PCA, and (ii) learning a linear regression model between the reduced input and each single vertex. Notably, the learned reduced rank estimators effectively capture a (low-dimensional) mapping that links visual stimuli to the full fMRI activity, whereas the baseline model, following a univariate approach, neglects any relationship among vertices. This is reflected in the results where our (simple) approach outperforms the baseline model and, in particular, R$^3$ model with Mat\'{e}rn kernel applied on the input space achieves the best performance, ranking 21st out of 94 (at the moment of submission) in the post-challenge leaderboard (see Table \ref{tab:neuro}). Moreover, the table reveals that for the same kernel choice, R$^3$ performs only slightly better than R$^4$ but at a significantly higher computational cost (see Fig. \ref{fig:neuroscience1}). Given the high dimensionality of the problem, both in terms of sample size and number of features, we implemented R$^3$ and R$^4$ algorithms using \texttt{PyTorch}~\cite{Pytorch2019} package to exploit GPU acceleration. This translates into a dramatic speedup of about three orders of magnitude (see Fig. \ref{fig:neuroscience1}, right panel).

{\small
\begin{figure}[t!]
\centering
\includegraphics[scale=1.0]{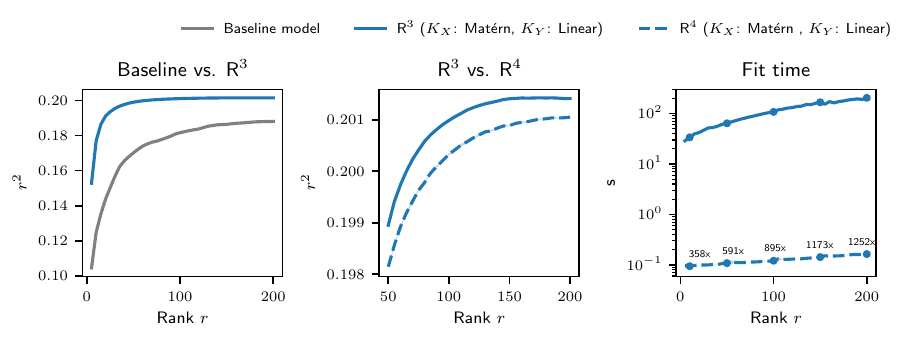}
    \caption{Assessment of the performance in predicting fMRI activity in terms of squared correlation, $r^2$ (the higher the better), and fit time. We show $r^2$ for baseline and R$^3$ (left), R$^3$ and R$^4$ (center), and the fit time for R$^3$ and R$^4$ (right). %
    The estimated quantities are represented by the average over 10 independent seed initializations.}\label{fig:neuroscience1}
\end{figure}
}

\subsection*{Noisy logistic map}
In the context of dynamical systems, our methodology can be employed to learn a low-rank approximation of the Koopman operator~\cite{Kostic2022} associated to a dynamical system. We demonstrate how R$^4$ can achieve comparable performance to R$^3$ in approximating the Koopman operator from data, with a specific focus on eigenvalue estimation. We analyze the noisy logistic map, defined as $x_{t=1}=(4x_t(1-x_t)+\xi_t) \; \text{mod 1}$ where $\X=[0,1]$ and  $\xi_t$ is an i.i.d. trigonometric noise \cite{Ostruszka2000}. For the numerical experiment, we generate trajectories of 5000 and 10000 points as validation and test sets, respectively, with parameter $N$ of the trigonometric noise set to 20. We set the target rank $r=3$, the oversampling parameter $s=20$, number of powering iterations $p=1$ and regularization $\gamma=10^{-7}$.

In Fig. \ref{fig:logisticmap} (left panel), we show the three leading eigenvalues $\lambda = \{1, -0.193\pm0.191i\}$ characterizing the Koopman operator of the noisy logistic map, and the eigenvalues estimated from R$^3$ and R$^4$. To quantitatively assess whether the eigenvalues of the logistic map are well estimated, we used the \textit{directed} Hausdorff distance (DHD) defined as $h(P, R) = \max_{p \in P} \min_{r \in R} |p - r|$, measuring the distance between a set $P$ and a reference set $R$. We show that the DHD between the estimated eigenvalues with R$^3$ and the true eigenvalues tend to zero as the training sample size increases (center panel). Remarkably, the estimated eigenvalues with R$^4$ are virtually the same as the ones estimated with R$^3$ (right panel).

{\small
\begin{figure}[t!]
\centering
\includegraphics[scale=1.0]{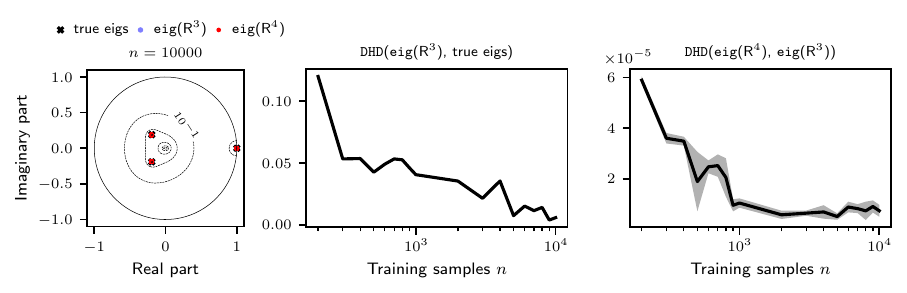}
    \caption{True and estimated eigenvalues of the logistic map (left), DHD between R$^3$ estimated eigenvalues and true eigenvalues (center), and DHD between R$^3$ and R$^4$ estimated eigenvalues.}\label{fig:logisticmap}
\end{figure}
}

\section{Conclusions}
We devised an efficient method to solve vector-valued regression problems by learning reduced rank estimators via a randomized algorithm. The numerical experiments confirmed the theoretical findings showing how R$^4$ virtually attains the same accuracy of classical reduced rank regression being dramatically faster when dealing with real-world, large scale datasets. 

An important open question for future research is how to combine the results of this work with statistical learning rates. This would unveil to which extent the computational complexity of reduced rank regression can be reduced without deterioration of its generalization properties. Other interesting paths include combining our results with different  strategies, such as Nystr\"{o}m subsampling~\cite{Rudi2015} or random features~\cite{Rahimi2007}, to further reduce the computational complexity of estimators, as well as addressing other vector-valued regression problems beyond the square loss.

\bibliographystyle{siamplain}
\bibliography{references}

\end{document}